\documentclass[10pt,a4paper]{extarticle}
\usepackage{graphicx} %

\usepackage[round]{natbib}
\renewcommand{\cite}{\citep}

\usepackage{fullpage}
\usepackage{libertine}
\usepackage{authblk}

\usepackage{amsmath}
\usepackage{amsthm}
\usepackage{amssymb}
\usepackage{color}
\usepackage{bm}
\usepackage{nicefrac}
\usepackage{xcolor}
\usepackage[bb=dsserif]{mathalpha}
\usepackage[group-separator={,},group-minimum-digits={3}]{siunitx}

\usepackage[pdftex,bookmarksnumbered,bookmarksopen,colorlinks,citecolor={blue!70!black},linkcolor={blue!70!black},urlcolor=blue]{hyperref}
\usepackage[capitalize]{cleveref}

\DeclareMathOperator{\E}{\mathbb{E}}
\newcommand{\sD}{\mathbb{D}}
\newcommand{\sR}{\mathbb{R}}

\newcommand{\define}{~\smash{\triangleq}~}

\definecolor{seabornblue}{HTML}{1f77b4}
\definecolor{seabornorange}{HTML}{ff7f0e}
\definecolor{seaborngreen}{HTML}{2ca02c}
\definecolor{seabornred}{HTML}{d62728}

\definecolor{seqA}{HTML}{9ecae1}  %
\definecolor{seqB}{HTML}{6baed6}  %
\definecolor{seqC}{HTML}{2171b5}  %
\definecolor{seqD}{HTML}{08519c}  %

\newtheorem{definition}{Definition}[section]
\newtheorem{proposition}{Proposition}[section]

\newcommand{\tpr}{\mathrm{TPR}}
\newcommand{\fpr}{\mathrm{FPR}}

\title{\vspace{-2em} %
On Choosing the $\mu$ Parameter in  Gaussian Differential Privacy}
\author[1]{Bogdan Kulynych}
\author[2]{Antti Honkela}
\affil[1]{{\footnotesize Biomedical Data Science Center, Lausanne University Hospital}}
\affil[2]{{\footnotesize University of Helsinki}}
\date{}

\usepackage{pgfplots}
\pgfplotsset{compat=1.18}
\usepackage{subcaption}
\usepackage{booktabs}
\usetikzlibrary{arrows.meta}

\begin{document}

\maketitle
\begin{abstract}
\noindent
Recent work argues for using Gaussian differential privacy (GDP) to report the privacy guarantees in privacy-preserving machine learning. We provide principled mappings from pure-DP $\varepsilon$ to GDP $\mu$ by matching the worst-case success of a strong-adversary membership inference attack in terms of three metrics: multiplicative advantage at fixed FPR, precision at fixed recall, and the standard privacy profile. We tabulate $\mu$ values across a useful range of parameters and recommend $\mu \approx \varepsilon/5$ as a conservative general-purpose conversion.
\end{abstract}

\section{Introduction}
\label{sec:intro}

Over the past decade, a number of differential privacy (DP) variants have emerged, with different ways of measuring the similarity of mechanism output distributions on adjacent datasets~\cite{desfontaines2020sok}.
Yet, the final guarantees are almost always reported in terms of $(\varepsilon, \delta)$ of approximate DP.
Recently, a subset of authors~\cite{gomez2026gaussian} noted that approximate DP guarantees provide a very inaccurate representation of privacy guarantees of many important algorithms, especially in machine learning, and propose using Gaussian DP (GDP) \citep{dong2019gaussian} for reporting and for defining the privacy budget instead.

In this note, we consider the problem of using $\mu$-GDP for defining the privacy budget in a general-purpose way, and ask the question of what is a sensible $\mu$ value in different scenarios.

\begin{figure}[p]
    \centering
    \begin{subfigure}[b]{0.3\linewidth}
        \centering
        \resizebox{!}{0.79\linewidth}{%
        \begin{tikzpicture}
            \def\mufig{1}
            \def\alphaarrow{0.001}
            \def\tprarrow{0.018298}%
            \pgfmathsetmacro{\ymaxval}{\tprarrow*1.15}
            \pgfmathsetmacro{\nodemidy}{(\tprarrow+\alphaarrow)/2}
            \pgfmathsetmacro{\nodexpos}{1.1* \alphaarrow }
            \begin{axis}[
                width=5cm, height=5cm,
                xlabel={FPR},
                ylabel={TPR},
                xmode=log,
                xmin=1e-5, xmax=1e-2, ymin=0, ymax=\ymaxval,
                axis line style={draw opacity=0.5},
                enlargelimits=false,
                font={\sffamily},
                xtick={1e-5,1e-4,1e-3,1e-2},
                ytick={0.01, 0.02},
                yticklabels={0.01, 0.02},
                scaled y ticks=false,
                grid=major,
                grid style={gray!25, thin},
                legend style={draw=none, fill=none, font={\sffamily\scriptsize},
                              at={(0.03,0.97)}, anchor=north west,
                              row sep=-2pt, inner sep=2pt},
                legend cell align=left,
            ]
            \addplot[color=seabornorange, thick] table[x=fpr, y=tpr] {data/plots/fig1a_tradeoff_mu1.dat};
            \addlegendentry{GDP}
            \addplot[color=black!55, dashed, thick,
                     domain=1e-5:0.1, samples=2]
                {x};
            \addlegendentry{TPR = FPR}

            \draw[{Stealth[length=4pt]}-{Stealth[length=4pt]},
                  gray!75, thick]
                (axis cs:\alphaarrow, \alphaarrow) -- (axis cs:\alphaarrow, \tprarrow);
            \node[anchor=west, font=\sffamily\footnotesize, align=left]
                at (axis cs:\nodexpos, \nodemidy)
                {$\dfrac{\tpr}{\fpr}$};
            \end{axis}
        \end{tikzpicture}}
        \caption{Log. mult. advantage for $\mu = 1$}
        \label{fig:tradeoff}
    \end{subfigure}%
    \hfill
    \begin{subfigure}[b]{0.3\linewidth}
        \centering
        \resizebox{!}{0.81\linewidth}{%
        \begin{tikzpicture}
            \begin{axis}[
                width=5cm, height=5cm,
                xlabel={$\varepsilon$}, ylabel={$\mu$},
                xmin=0, xmax=4, ymin=0, ymax=4,
                axis line style={draw opacity=0.5},
                enlargelimits=false,
                font={\sffamily},
                xtick={0,1,2,3,4}, ytick={0,1,2,3,4},
                grid=both,
                grid style={gray!25, thin},
                legend style={draw=none, fill=none, font={\sffamily\footnotesize},
                              at={(0.03,0.97)}, anchor=north west,
                              row sep=-2pt, inner sep=2pt},
                legend cell align=left,
            ]
            \addplot[color=gray!50, thick, dashed, domain=0:5, samples=2]
                {1/3 * x};
            \addlegendentry{$\nicefrac{\varepsilon}{3}$}
            \addplot[color=gray!30, thick, dashed, domain=0:5, samples=2]
                {1/5 * x};
            \addlegendentry{$\nicefrac{\varepsilon}{5}$}

            \addplot[color=seqA, thick] table[x=eps, y=cA] {data/plots/fig1b_mu_vs_eps_multadv.dat};
            \addlegendentry{$\alpha = 10^{-5}$}
            \addplot[color=seqB, thick] table[x=eps, y=cB] {data/plots/fig1b_mu_vs_eps_multadv.dat};
            \addlegendentry{$\alpha = 10^{-4}$}
            \addplot[color=seqC, thick] table[x=eps, y=cC] {data/plots/fig1b_mu_vs_eps_multadv.dat};
            \addlegendentry{$\alpha = 10^{-3}$}
            \addplot[color=seqD, thick] table[x=eps, y=cD] {data/plots/fig1b_mu_vs_eps_multadv.dat};
            \addlegendentry{$\alpha = 10^{-2}$}

            \end{axis}
        \end{tikzpicture}}
        \caption{Mapping via mult. adv. at FPR $\alpha$}
        \label{fig:mu-calibration}
    \end{subfigure}%
    \hfill%
    \begin{subfigure}[b]{0.35\linewidth}
        \centering
        {\sffamily\footnotesize
        \setlength{\tabcolsep}{3pt}
        \renewcommand{\arraystretch}{1.18}
        \begin{tabular}{@{}rcccccc@{}}
    & \multicolumn{6}{c}{$\varepsilon$} \\
    \cmidrule(l){2-7}
    $\alpha$ & 0.1 & 0.2 & 0.5 & 1.0 & 1.5 & 2.0 \\
    \cmidrule(l){1-1} \cmidrule(l){2-7}
    $10^{-2}$ & 0.038 & 0.076 & 0.19 & 0.4 & 0.63 & 0.88 \\
    $10^{-3}$ & 0.03 & 0.06 & 0.15 & 0.31 & 0.48 & 0.65 \\
    $10^{-4}$ & 0.025 & 0.051 & 0.13 & 0.26 & 0.4 & 0.54 \\
    $10^{-5}$ & 0.022 & 0.045 & 0.11 & 0.23 & 0.35 & 0.47 \\
\end{tabular}
}
        \vspace{1.5em}
        \caption{Mapping via mult. adv. at FPR $\alpha$}
        \label{tab:mu-values}
    \end{subfigure}
    \caption{Mapping between pure DP $\varepsilon$ and GDP $\mu$ via
    multiplicative-advantage semantics at FPR level $\alpha$.
    \textbf{\subref{fig:tradeoff}}:~Illustration of multiplicative advantage for GDP with $\mu = 1$ at $\alpha = 10^{-3}$ (gray arrow).
    \textbf{\subref{fig:mu-calibration}}, \textbf{\subref{tab:mu-values}}:~GDP $\mu$ values calibrated to multiplicative advantage for different FPR levels
    $\alpha$.}
    \label{fig:dp-gdp-mapping}
\end{figure}

\section{Setup}
\label{sec:setup}

Suppose a curator releases the output $\theta \sim M(S)$, e.g., a statistical computation, machine learning model, or a synthetic dataset, with $\theta \in \Theta$, of a randomized
\emph{mechanism} applied to a sensitive dataset $S \in 2^\sD$ over some data space $\sD$.

\paragraph{Strong-adversary membership inference.} We follow the standard hypothesis-testing approach to characterizing the privacy risk carried by the release~\cite{wasserman2010statistical,kairouz2015composition,dong2019gaussian}. Specifically, we use a security-first interpretation of the hypothesis-testing setup, in which we formalize the privacy risk by modeling a so-called \emph{strong adversary}~\cite[see, e.g.,][]{cuff2016differential, kulynych2024attack}. The adversary has knowledge of the mechanism $M(\cdot)$, a partial dataset $\bar S \in \sD^{n - 1}$, and aims to infer whether a given target record $z \in \sD$ was part of the input dataset that produced an observed output of the mechanism $\theta$. To mount this \emph{membership inference attack} (MIA), the adversary runs a test $\phi: \Theta \rightarrow [0, 1]$, where $\phi(\theta) = 1$ indicates the guess that the record was used as part of training, and $\phi(\theta) = 0$ indicates otherwise.

\paragraph{Attack success.}
We can characterize the adversary's success using the \emph{true-positive rate} (TPR), or \emph{recall}, and \emph{false-positive rate} (FPR) of the MIA:
\begin{equation}
    \tpr_\phi \define \E_{\theta \sim M(\bar S \cup \{z\})}[\phi], \quad
    \fpr_\phi \define \E_{\theta \sim M(\bar S)}[\phi].
\end{equation}
In the context of testing, the TPR corresponds to the test's \emph{power} or inverse of the Type II error, and FPR corresponds to the test's Type I error. To represent the attack success number compactly, a common approach is to use either additive advantage $\tpr_\phi - \fpr_\phi$~\cite[see, e.g.,][]{yeom2018privacy}, or \emph{multiplicative advantage} $\nicefrac{\tpr_\phi}{\fpr_\phi}$. In this note, we focus on the latter, as it is more robust in the small-FPR regime. In a slight abuse of notation, we drop the subscript where appropriate.

\paragraph{Differential privacy variants.} Within this view, the classical notion of differential privacy (DP)~\cite{dwork2006calibrating,dwork2014algorithmic}, also called pure DP, ensures that the multiplicative advantage is always bounded:
\begin{definition}
    A mechanism $M(\cdot)$ satisfies $\varepsilon$-DP, where $\varepsilon \in \sR$, if for any test function $\phi: \Theta \rightarrow [0, 1]$, we have:
    \begin{equation}
        e^{-\varepsilon} \leq \frac{\tpr_\phi}{\fpr_\phi} \leq e^\varepsilon.
    \end{equation}
\end{definition}
Another common variant is approximate differential privacy (ADP), which relaxes the bounds:
\begin{definition}
    A mechanism $M(\cdot)$ satisfies $(\varepsilon, \delta)$-DP or $(\varepsilon, \delta)$-ADP, where $\varepsilon \ge 0$, $\delta \in [0, 1]$, if for any test function $\phi: \Theta \rightarrow [0, 1]$, we have:
    \begin{equation}
        \tpr_\phi \leq e^\varepsilon \fpr_\phi + \delta
        \quad\text{and}\quad
        \fpr_\phi \leq e^\varepsilon \tpr_\phi + \delta.
    \end{equation}
\end{definition}
Any ADP mechanism satisfies a potentially infinite collection of $(\varepsilon, \delta)$-DP guarantees, characterized by the \emph{privacy profile} functions $\varepsilon: [0, 1] \to \sR$ that associates any given $\delta$ to $\varepsilon$ such that the mechanism satisfies $(\varepsilon, \delta)$-DP~\cite{balle2018privacy}, and the analogous $\delta: \sR \to [0, 1]$.

Another approach called $f$-differential privacy ($f$-DP)~\cite{dong2019gaussian} instead parameterizes the attack success using bounds on the set of achievable TPR-FPR pairs, or, in other words, bounds on the receiver-operating characteristic (ROC) curve of the worst-case MIA:
\begin{definition}
    A mechanism $M(\cdot)$ satisfies $f$-DP if for any test function $\phi: \Theta \rightarrow [0, 1]$, we have:
    \begin{equation}
        \tpr_\phi \leq 1 - f(\fpr_\phi), \quad\text{and}\quad \fpr_\phi \leq 1 - f(\tpr_\phi),
    \end{equation}
    where $f: [0, 1] \rightarrow [0, 1]$ is a given convex, non-increasing function called the \emph{trade-off curve}.
\end{definition}
We can losslessly convert between a trade-off curve $f$ and a privacy profile $\varepsilon(\delta)$. Thus, both contain equivalent information about the privacy guarantees of the mechanism, albeit with different immediate interpretations. Both representations are also potentially information-theoretically complete in the sense that they could contain all the information there is to characterize the success rate of the worst-case MIA~\cite{gomez2026gaussian}.

Gaussian differential privacy (GDP)~\cite{dong2019gaussian} is a guarantee that mounting a MIA against the mechanism outputs is as hard as distinguishing whether a single sample came from a standard normal distribution $\mathcal{N}(0, 1)$ or from its version $\mathcal{N}(\mu, 1)$ shifted by $\mu$ units:
\begin{definition}
    A mechanism $M(\cdot)$ satisfies $\mu$-GDP if it satisfies $f_\mu$-DP with $f_\mu$ given by:
    \begin{equation}
        f_\mu(\alpha) = \Phi(\Phi^{-1}(1 - \alpha) - \mu),
    \end{equation}
    where $\Phi(x)$ is the cumulative distribution function of the standard normal distribution.
\end{definition}
Many practical algorithms used in privacy-preserving machine learning and statistics are well-characterized by GDP~\cite{gomez2026gaussian}.

\section{Mapping Pure DP $\varepsilon$ to Gaussian DP $\mu$}
Our goal is to find a general-purpose correspondence between $\varepsilon$-DP, for which standard guidelines and recommendations already exist, and GDP parameter values $\mu$. Unfortunately, there is no straightforward way to do so, as the mapping could be done in a multitude of ways. We present several ways to do so that are consistent with prior approaches to interpret the semantics of DP protection.

\paragraph{General approach.} We define several notions of \emph{privacy risk}. Then, for each, we establish a correspondence between the DP and GDP parameters as follows. First, we fix a value of $\varepsilon \geq 0$. We find the $\mu$ parameter that attains the same value of risk. This way, our mapping ensures that each GDP $\mu$ parameter attains the same semantics as a pure-DP mechanism with the corresponding $\varepsilon$.

\paragraph{Multiplicative advantage.}
First, we introduce a notion of risk based on the multiplicative advantage of MIAs:
\begin{definition}
    We define the logarithmic multiplicative advantage $\varepsilon$ at FPR level $\alpha \in (0, 1)$ as follows:
    \begin{equation}
        \varepsilon \define \sup_{\phi:\, \fpr_\phi \leq \alpha}\,
            \max\!\left(\log\frac{\tpr_\phi}{\fpr_\phi},\, \log\frac{\fpr_\phi}{\tpr_\phi}\right).
    \end{equation}
\end{definition}
For any pure-DP mechanism, the multiplicative advantage at any FPR level $\alpha \in (0, 1)$ is at most $\varepsilon$ by definition.
Next, we provide closed-form solutions for the $\mu$ parameter values corresponding to a given multiplicative advantage and FPR level.
\begin{proposition}\label{prop:mu-calibration}
    For $\alpha \in (0,1)$ and $\varepsilon \geq 0$, $\mu$-GDP ensures the multiplicative advantage bound of $e^\varepsilon$ at FPR level $\alpha$ if:
    \begin{equation}\label{eq:mu-calibration}
        \mu = \begin{cases}
            \Phi^{-1}(e^\varepsilon \alpha) - \Phi^{-1}(\alpha)
                & \alpha < \tfrac{1}{e^\varepsilon+1}, \\[4pt]
            \Phi^{-1}(1-\alpha) - \Phi^{-1}\!\bigl(e^{-\varepsilon}(1-\alpha)\bigr)
                & \alpha \geq \tfrac{1}{e^\varepsilon+1}.
        \end{cases}
    \end{equation}
\end{proposition}
The proof is in the Appendix.
In \cref{fig:dp-gdp-mapping}, we show the results of this mapping for a useful range of $\varepsilon$ and $\alpha$ values, and defer to \cref{prop:mu-calibration} for mapping the parameters outside of this range.

\paragraph{Precision at fixed recall.}
As another standard metric to evaluate MIAs, we consider precision:
\begin{definition}
    We define precision as follows:
    \begin{equation}
        \mathtt{precision}_{\phi} \define \frac{\tpr_\phi}{\tpr_\phi + \fpr_\phi}.
    \end{equation}
\end{definition}
For any pure-DP mechanism, we have $\mathtt{precision}_{\phi} \leq e^\varepsilon / (e^\varepsilon + 1)$ for any test $\phi$.
Next, we provide a closed-form solution for GDP:
\begin{proposition}\label{prop:mu-recall}
    For recall level $\bar\beta \in (0, 1)$ and $\varepsilon \geq 0$, $\mu$-GDP ensures the same
    precision as $\varepsilon$-DP if:
    \begin{equation}\label{eq:mu-recall}
        \mu = \Phi^{-1}(\bar\beta) - \Phi^{-1}\!\bigl(e^{-\varepsilon}\bar\beta\bigr).
    \end{equation}
\end{proposition}
The proof is in the Appendix.
In \cref{fig:dp-gdp-mapping-recall} we show the results of this mapping for a useful range of $\varepsilon$ and $\bar\beta$ values, and defer to \cref{prop:mu-recall} for mapping the parameters outside of this range.

\begin{figure}[p]
    \centering
    \begin{subfigure}[b]{0.3\linewidth}
        \centering
        \resizebox{!}{0.79\linewidth}{%
        \begin{tikzpicture}
            \def\mufig{1}
            \begin{axis}[
                width=5cm, height=5cm,
                xlabel={TPR},
                ylabel={Precision},
                xmin=0, xmax=1, ymin=0.4, ymax=1.05,
                axis line style={draw opacity=0.5},
                enlargelimits=false,
                font={\sffamily},
                xtick={0, 0.5, 1}, ytick={0.5, 0.75, 1},
                grid=major,
                grid style={gray!25, thin},
                legend style={draw=none, fill=none, font={\sffamily\scriptsize},
                              at={(0.97,0.03)}, anchor=south east,
                              row sep=-2pt, inner sep=2pt},
                legend cell align=left,
            ]
            \addplot[color=seabornorange, thick] table[x=recall, y=precision] {data/plots/fig2a_precision_mu1.dat};
            \addlegendentry{GDP}
            \addplot[color=black!55, dashed, thick,
                     domain=0:1, samples=2]
                {0.5};
            \addlegendentry{$\mathrm{Precision} = \nicefrac{1}{2}$}
            \end{axis}
        \end{tikzpicture}}
        \caption{Precision vs.\ recall $\hat\beta$ at $\mu = 1$}
        \label{fig:tradeoff-recall}
    \end{subfigure}%
    \hfill
    \begin{subfigure}[b]{0.3\linewidth}
        \centering
        \resizebox{!}{0.81\linewidth}{%
        \begin{tikzpicture}
            \begin{axis}[
                width=5cm, height=5cm,
                xlabel={$\varepsilon$}, ylabel={$\mu$},
                xmin=0, xmax=4, ymin=0, ymax=4,
                axis line style={draw opacity=0.5},
                enlargelimits=false,
                font={\sffamily},
                xtick={0,1,2,3,4}, ytick={0,1,2,3,4},
                grid=both,
                grid style={gray!25, thin},
                legend style={draw=none, fill=none, font={\sffamily\footnotesize},
                              at={(0.03,0.97)}, anchor=north west,
                              row sep=-2pt, inner sep=2pt},
                legend cell align=left,
            ]
            \addplot[color=gray!50, thick, dashed, domain=0:5, samples=2]
                {1/3 * x};
            \addlegendentry{$\nicefrac{\varepsilon}{3}$}
            \addplot[color=gray!30, thick, dashed, domain=0:5, samples=2]
                {1/5 * x};
            \addlegendentry{$\nicefrac{\varepsilon}{5}$}

            \addplot[color=seqA, thick] table[x=eps, y=cA] {data/plots/fig2b_mu_vs_eps_recall.dat};
            \addlegendentry{$\bar\beta = 10^{-5}$}
            \addplot[color=seqB, thick] table[x=eps, y=cB] {data/plots/fig2b_mu_vs_eps_recall.dat};
            \addlegendentry{$\bar\beta = 10^{-4}$}
            \addplot[color=seqC, thick] table[x=eps, y=cC] {data/plots/fig2b_mu_vs_eps_recall.dat};
            \addlegendentry{$\bar\beta = 10^{-3}$}
            \addplot[color=seqD, thick] table[x=eps, y=cD] {data/plots/fig2b_mu_vs_eps_recall.dat};
            \addlegendentry{$\bar\beta = 10^{-2}$}

            \end{axis}
        \end{tikzpicture}}
        \caption{Mapping, via prec. at recall $\hat\beta$}
        \label{fig:mu-calibration-recall}
    \end{subfigure}%
    \hfill%
    \begin{subfigure}[b]{0.35\linewidth}
        \centering
        {\sffamily\footnotesize
        \setlength{\tabcolsep}{3pt}
        \renewcommand{\arraystretch}{1.18}
        \begin{tabular}{@{}rcccccc@{}}
    & \multicolumn{6}{c}{$\varepsilon$} \\
    \cmidrule(l){2-7}
    $\bar\beta$ & 0.1 & 0.2 & 0.5 & 1.0 & 1.5 & 2.0 \\
    \cmidrule(l){1-1} \cmidrule(l){2-7}
    $10^{-2}$ & 0.037 & 0.074 & 0.18 & 0.35 & 0.52 & 0.67 \\
    $10^{-3}$ & 0.03 & 0.059 & 0.15 & 0.29 & 0.42 & 0.55 \\
    $10^{-4}$ & 0.025 & 0.05 & 0.12 & 0.25 & 0.36 & 0.48 \\
    $10^{-5}$ & 0.022 & 0.044 & 0.11 & 0.22 & 0.32 & 0.43 \\
\end{tabular}
}
        \vspace{1.5em}
        \caption{Mapping via prec. at recall $\hat\beta$}
        \label{tab:mu-values-recall}
    \end{subfigure}
    \caption{Mapping between pure DP $\varepsilon$ and GDP $\mu$ via
    precision at recall level $\tpr = \bar\beta$.
    \textbf{\subref{fig:tradeoff-recall}}:~Precision under $\mu = 1$ as a function of recall $\bar\beta$ (gray arrow shows the trivial baseline $\nicefrac{1}{2}$ at $\bar\beta = \nicefrac{1}{2}$).
    \textbf{\subref{fig:mu-calibration-recall}}, \textbf{\subref{tab:mu-values-recall}}:~GDP $\mu$ values calibrated to precision at recall for different recall levels $\bar\beta$.}
    \label{fig:dp-gdp-mapping-recall}
\end{figure}

\paragraph{Approximate DP.}
The most common way to summarize ADP guarantees is via the privacy profile $\varepsilon(\delta)$. In the language of MIA error rates, it is defined as follows:
\begin{definition}
    We define the privacy profile $\varepsilon: [0, 1] \to \sR$ as follows:
    \begin{equation}
        \varepsilon(\delta) \define \sup_{\phi:~\Theta \to [0, 1]} \max\left( \log \frac{\tpr_\phi - \delta}{\fpr_\phi},\;\log \frac{\fpr_\phi - \delta}{\tpr_\phi}\right).
    \end{equation}
\end{definition}
For $\mu$-GDP, the inverse profile $\delta(\varepsilon)$ admits the closed form~\cite{dong2019gaussian}:
\begin{equation}\label{eq:gdp-profile}
    \delta_\mu(\varepsilon) = \Phi\!\Bigl(\!-\frac{\varepsilon}{\mu} + \frac{\mu}{2}\Bigr) - e^\varepsilon\,\Phi\!\Bigl(\!-\frac{\varepsilon}{\mu} - \frac{\mu}{2}\Bigr).
\end{equation}
Given a target $(\varepsilon, \delta)$ pair, the corresponding $\mu$ is the unique solution of $\delta_\mu(\varepsilon) = \delta$. The map $\mu \mapsto \delta_\mu(\varepsilon)$ is monotone increasing for fixed $\varepsilon$, so we can invert it numerically using bisection.
In \cref{fig:dp-gdp-mapping-adp} we show the resulting numerical mapping for a useful range of $\varepsilon$ and $\delta$ values.

\medskip

In \cref{app:risk-vs-mu}, we also provide conversions from GDP directly to all of the notions of risk above for a selection of realistic values of $\mu$.

\begin{figure}[p]
    \centering
    \begin{subfigure}[b]{0.3\linewidth}
        \centering
        \resizebox{!}{0.79\linewidth}{%
        \begin{tikzpicture}
            \def\mufig{1}
            \begin{axis}[
                width=5cm, height=5cm,
                xlabel={$\delta$},
                ylabel={$\varepsilon$},
                xmin=-10, xmax=0, ymin=0, ymax=8,
                axis line style={draw opacity=0.5},
                enlargelimits=false,
                font={\sffamily},
                xtick={-9,-6,-3,0},
                xticklabels={$10^{-9}$,$10^{-6}$,$10^{-3}$,$1$},
                ytick={0,2,4,6,8},
                grid=major,
                grid style={gray!25, thin},
                legend style={draw=none, fill=none, font={\sffamily\scriptsize},
                              at={(0.97,0.97)}, anchor=north east,
                              row sep=-2pt, inner sep=2pt},
                legend cell align=left,
            ]
            \addplot[color=seabornorange, thick] table[x=log10delta, y=eps] {data/plots/fig3a_log10delta_mu1.dat};
            \addlegendentry{GDP}
            \end{axis}
        \end{tikzpicture}}
        \caption{Privacy profile, $\varepsilon(\delta)$ at $\mu = 1$}
        \label{fig:profile-illustration}
    \end{subfigure}%
    \hfill
    \begin{subfigure}[b]{0.3\linewidth}
        \centering
        \resizebox{!}{0.81\linewidth}{%
        \begin{tikzpicture}
            \begin{axis}[
                width=5cm, height=5cm,
                xlabel={$\varepsilon$}, ylabel={$\mu$},
                xmin=0, xmax=4, ymin=0, ymax=4,
                axis line style={draw opacity=0.5},
                enlargelimits=false,
                font={\sffamily},
                xtick={0,1,2,3,4}, ytick={0,1,2,3,4},
                grid=both,
                grid style={gray!25, thin},
                legend style={draw=none, fill=none, font={\sffamily\footnotesize},
                              at={(0.03,0.97)}, anchor=north west,
                              row sep=-2pt, inner sep=2pt},
                legend cell align=left,
            ]
            \addplot[color=gray!50, thick, dashed, domain=0:5, samples=2]
                {1/3 * x};
            \addlegendentry{$\nicefrac{\varepsilon}{3}$}
            \addplot[color=gray!30, thick, dashed, domain=0:5, samples=2]
                {1/5 * x};
            \addlegendentry{$\nicefrac{\varepsilon}{5}$}

            \addplot[color=seqA, thick] table[x=eps, y=cA] {data/plots/fig3b_mu_vs_eps_adp.dat};
            \addlegendentry{$\delta = 10^{-5}$}
            \addplot[color=seqB, thick] table[x=eps, y=cB] {data/plots/fig3b_mu_vs_eps_adp.dat};
            \addlegendentry{$\delta = 10^{-6}$}
            \addplot[color=seqC, thick] table[x=eps, y=cC] {data/plots/fig3b_mu_vs_eps_adp.dat};
            \addlegendentry{$\delta = 10^{-7}$}
            \addplot[color=seqD, thick] table[x=eps, y=cD] {data/plots/fig3b_mu_vs_eps_adp.dat};
            \addlegendentry{$\delta = 10^{-8}$}
            \end{axis}
        \end{tikzpicture}}
        \caption{Mapping, via privacy profile}
        \label{fig:mu-calibration-adp}
    \end{subfigure}%
    \hfill%
    \begin{subfigure}[b]{0.35\linewidth}
        \centering
        {\sffamily\footnotesize
        \setlength{\tabcolsep}{3pt}
        \renewcommand{\arraystretch}{1.18}
        \begin{tabular}{@{}rcccccc@{}}
    & \multicolumn{6}{c}{$\varepsilon$} \\
    \cmidrule(l){2-7}
    $\delta$ & 0.1 & 0.2 & 0.5 & 1.0 & 1.5 & 2.0 \\
    \cmidrule(l){1-1} \cmidrule(l){2-7}
    $10^{-5}$ & 0.033 & 0.061 & 0.14 & 0.27 & 0.39 & 0.5 \\
    $10^{-6}$ & 0.028 & 0.053 & 0.12 & 0.24 & 0.34 & 0.45 \\
    $10^{-7}$ & 0.024 & 0.047 & 0.11 & 0.21 & 0.31 & 0.41 \\
    $10^{-8}$ & 0.022 & 0.042 & 0.1 & 0.2 & 0.29 & 0.38 \\
\end{tabular}
}
        \vspace{1.5em}
        \caption{Mapping, via privacy profile}
        \label{tab:mu-values-adp}
    \end{subfigure}
    \caption{Mapping between pure DP $\varepsilon$ and GDP $\mu$ via the
    privacy profile of $\mu$-GDP at fixed $\delta$.
    \textbf{\subref{fig:profile-illustration}}:~Privacy profile $\varepsilon(\delta)$ for $\mu = 1$.
    \textbf{\subref{fig:mu-calibration-adp}}, \textbf{\subref{tab:mu-values-adp}}:~GDP $\mu$ values calibrated via the standard $(\varepsilon, \delta)$ derivation for GDP with different $\delta$ levels.}
    \label{fig:dp-gdp-mapping-adp}
\end{figure}

\section{Conclusion}

The $\mu$-GDP guarantees can be interpreted in many ways, depending on the relevant notion of risk and whether the focus is on high- or low- probability events. Reflecting on the values of the conversions we observe in \cref{fig:dp-gdp-mapping,fig:dp-gdp-mapping-recall,fig:dp-gdp-mapping-adp}, we can see some patterns that hold for all risk notions. In comparison with more familiar $\varepsilon$-DP guarantees, $\mu = \varepsilon / 5$ provides a conservative conversion rule for general-purpose use.
For settings where low-probability events can be ignored, a less conservative conversion such as $\mu = \varepsilon / 3$ could be applied.

\clearpage

\bibliographystyle{unsrtnat}
\bibliography{main}

\begin{thebibliography}{11}
\providecommand{\natexlab}[1]{#1}
\providecommand{\url}[1]{\texttt{#1}}
\expandafter\ifx\csname urlstyle\endcsname\relax
  \providecommand{\doi}[1]{doi: #1}\else
  \providecommand{\doi}{doi: \begingroup \urlstyle{rm}\Url}\fi

\bibitem[Desfontaines and Pej{\'o}(2020)]{desfontaines2020sok}
Damien Desfontaines and Bal{\'a}zs Pej{\'o}.
\newblock Sok: Differential privacies.
\newblock \emph{Proceedings on Privacy Enhancing Technologies}, 2020\penalty0 (2):\penalty0 288--313, 2020.

\bibitem[Gomez et~al.(2026)Gomez, Kulynych, Kaissis, Calmon, Hayes, Balle, and Honkela]{gomez2026gaussian}
Juan~Felipe Gomez, Bogdan Kulynych, Georgios Kaissis, Flavio~P Calmon, Jamie Hayes, Borja Balle, and Antti Honkela.
\newblock Gaussian dp for reporting differential privacy guarantees in machine learning.
\newblock In \emph{IEEE SatML}, 2026.

\bibitem[Dong et~al.(2022)Dong, Roth, and Su]{dong2019gaussian}
Jinshuo Dong, Aaron Roth, and Weijie~J Su.
\newblock Gaussian differential privacy.
\newblock \emph{Journal of the Royal Statistical Society Series B: Statistical Methodology}, 2022.

\bibitem[Wasserman and Zhou(2010)]{wasserman2010statistical}
Larry Wasserman and Shuheng Zhou.
\newblock A statistical framework for differential privacy.
\newblock \emph{Journal of the American Statistical Association}, 2010.

\bibitem[Kairouz et~al.(2015)Kairouz, Oh, and Viswanath]{kairouz2015composition}
Peter Kairouz, Sewoong Oh, and Pramod Viswanath.
\newblock The composition theorem for differential privacy.
\newblock In \emph{International conference on machine learning}. PMLR, 2015.

\bibitem[Cuff and Yu(2016)]{cuff2016differential}
Paul Cuff and Lanqing Yu.
\newblock Differential privacy as a mutual information constraint.
\newblock In \emph{Proceedings of the 2016 ACM SIGSAC Conference on Computer and Communications Security}, pages 43--54, 2016.

\bibitem[Kulynych et~al.(2024)Kulynych, Gomez, Kaissis, Calmon, and Troncoso]{kulynych2024attack}
Bogdan Kulynych, Juan~Felipe Gomez, Georgios Kaissis, Flavio du~Pin Calmon, and Carmela Troncoso.
\newblock Attack-aware noise calibration for differential privacy.
\newblock \emph{Advances in Neural Information Processing Systems ({NeurIPS})}, 2024.

\bibitem[Yeom et~al.(2018)Yeom, Giacomelli, Fredrikson, and Jha]{yeom2018privacy}
Samuel Yeom, Irene Giacomelli, Matt Fredrikson, and Somesh Jha.
\newblock Privacy risk in machine learning: Analyzing the connection to overfitting.
\newblock In \emph{2018 IEEE 31st Computer Security Foundations Symposium (CSF)}. IEEE, 2018.

\bibitem[Dwork et~al.(2006)Dwork, McSherry, Nissim, and Smith]{dwork2006calibrating}
Cynthia Dwork, Frank McSherry, Kobbi Nissim, and Adam Smith.
\newblock Calibrating noise to sensitivity in private data analysis.
\newblock In \emph{Proceedings of the Theory of Cryptography Conference}, 2006.

\bibitem[Dwork and Roth(2014)]{dwork2014algorithmic}
Cynthia Dwork and Aaron Roth.
\newblock The algorithmic foundations of differential privacy.
\newblock \emph{Foundations and Trends in Theoretical Computer Science}, 2014.

\bibitem[Balle et~al.(2018)Balle, Barthe, and Gaboardi]{balle2018privacy}
Borja Balle, Gilles Barthe, and Marco Gaboardi.
\newblock Privacy amplification by subsampling: Tight analyses via couplings and divergences.
\newblock \emph{Advances in Neural Information Processing Systems ({NeurIPS})}, 2018.

\end{thebibliography}
\clearpage
\appendix

\section{Proofs}
\label{app:proofs}

\begin{proof}[Proof of \cref{prop:mu-calibration}]
We analyze separately for two regimes of $\alpha$ values corresponding to the segments of the piecewise-linear trade-off curve corresponding to DP, $f_\varepsilon(\alpha) = \max\{0, 1 - e^\varepsilon \alpha, e^{\varepsilon}(1 - \alpha)\}$.

\smallskip\noindent\textit{Case 1: $\alpha < 1/(e^\varepsilon+1)$.}
In this case, $f_\varepsilon(\alpha) = 1 - e^\varepsilon\alpha$.
Using $\Phi(-x) = 1 - \Phi(x)$ and $\Phi^{-1}(1-\alpha) = -\Phi^{-1}(\alpha)$:
\begin{equation}
    1 - f_\mu(\alpha)
    = 1 - \Phi\!\bigl(\Phi^{-1}(1-\alpha) - \mu\bigr)
    = \Phi\!\bigl(\Phi^{-1}(\alpha) + \mu\bigr).
\end{equation}
Setting $\mu = \Phi^{-1}(e^\varepsilon\alpha) - \Phi^{-1}(\alpha)$ gives $\Phi^{-1}(\alpha)+\mu = \Phi^{-1}(e^\varepsilon\alpha)$, so
$1 - f_\mu(\alpha) = e^\varepsilon\alpha = 1 - f_\varepsilon(\alpha)$.

\smallskip\noindent\textit{Case 2: $\alpha \geq 1/(e^\varepsilon+1)$.}
In this case, $f_\varepsilon(\alpha) = e^{-\varepsilon}(1-\alpha)$.
Directly from the definition of $f_\mu$:
\begin{equation}
    f_\mu(\alpha) = \Phi\!\bigl(\Phi^{-1}(1-\alpha) - \mu\bigr).
\end{equation}
Setting $\mu = \Phi^{-1}(1-\alpha) - \Phi^{-1}(e^{-\varepsilon}(1-\alpha))$ gives
$\Phi^{-1}(1-\alpha) - \mu = \Phi^{-1}(e^{-\varepsilon}(1-\alpha))$, so
$f_\mu(\alpha) = e^{-\varepsilon}(1-\alpha) = f_\varepsilon(\alpha)$.
\end{proof}

\begin{proof}[Proof of \cref{prop:mu-recall}]
The trade-off curve of $\mu$-GDP gives, for any test $\phi$,
$\fpr_\phi \geq \Phi(\Phi^{-1}(\tpr_\phi) - \mu)$. Therefore, for any $\phi$
with $\tpr_\phi \geq \bar\beta$,
\begin{equation}
    \frac{\tpr_\phi}{\fpr_\phi}
    \;\leq\; \frac{\tpr_\phi}{\Phi\!\bigl(\Phi^{-1}(\tpr_\phi) - \mu\bigr)}
    \;\leq\; \frac{\bar\beta}{\Phi\!\bigl(\Phi^{-1}(\bar\beta) - \mu\bigr)},
\end{equation}
where the second inequality uses that $t \mapsto t/\Phi(\Phi^{-1}(t) - \mu)$
is non-increasing on $(0, 1)$ for $\mu \geq 0$.
Setting the right-hand side equal to $e^\varepsilon$ gives
$\Phi(\Phi^{-1}(\bar\beta) - \mu) = e^{-\varepsilon}\bar\beta$, so
$\mu = \Phi^{-1}(\bar\beta) - \Phi^{-1}\!\bigl(e^{-\varepsilon}\bar\beta\bigr)$.
\end{proof}

\section{Mapping Gaussian DP $\mu$ to Risk}
\label{app:risk-vs-mu}

In the main body, we presented the maps from pure DP to the GDP parameter $\mu$ for each notion of risk. In this section, we present a mapping from $\mu$ to the values of risk.

\paragraph{Multiplicative advantage.}
\Cref{fig:app-multi-adv} shows the worst-case multiplicative advantage $\nicefrac{\tpr}{\fpr}$ for several $\mu$ values and FPR levels $\fpr = \alpha$.

\begin{figure}[h]
    \centering
    \begin{subfigure}[c]{0.55\linewidth}
        \centering
        \resizebox{0.9\linewidth}{!}{%
        \begin{tikzpicture}
            \begin{axis}[
                width=8cm, height=6cm,
                xlabel={TPR},
                ylabel={Multiplicative advantage $\nicefrac{\tpr}{\fpr}$},
                xmode=log, ymode=log,
                xmin=1e-6, xmax=1e-1, ymin=1, ymax=1e3,
                axis line style={draw opacity=0.5},
                enlargelimits=false,
                font={\sffamily},
                xtick={1e-6,1e-5,1e-4,1e-3,1e-2,1e-1},
                ytick={1,1e1,1e2,1e3},
                grid=major,
                grid style={gray!25, thin},
                legend style={draw=none, fill=none, font={\sffamily\footnotesize},
                              at={(0.97,0.97)}, anchor=north east,
                              row sep=-2pt, inner sep=2pt},
                legend cell align=left,
            ]
            \addplot[color=black!55, dashed, thick,
                     domain=1e-6:1e-1, samples=2]
                {1};
            \addlegendentry{TPR = FPR}
            \addplot[color=seqA, thick] table[x=fpr, y=cA] {data/plots/fig4a_advratio_vs_fpr.dat};
            \addlegendentry{$\mu = 0.1$}
            \addplot[color=seqB, thick] table[x=fpr, y=cB] {data/plots/fig4a_advratio_vs_fpr.dat};
            \addlegendentry{$\mu = 0.25$}
            \addplot[color=seqC, thick] table[x=fpr, y=cC] {data/plots/fig4a_advratio_vs_fpr.dat};
            \addlegendentry{$\mu = 0.5$}
            \addplot[color=seqD, thick] table[x=fpr, y=cD] {data/plots/fig4a_advratio_vs_fpr.dat};
            \addlegendentry{$\mu = 1$}
            \end{axis}
        \end{tikzpicture}}
        \caption{Multiplicative advantage vs.\ FPR}
        \label{fig:app-tradeoff-multi}
    \end{subfigure}%
    \hfill
    \begin{subfigure}[c]{0.42\linewidth}
        \centering
        {\sffamily\footnotesize
        \setlength{\tabcolsep}{4pt}
        \renewcommand{\arraystretch}{1.18}
        \begin{tabular}{@{}rcccccc@{}}
    & \multicolumn{6}{c}{$\mu$} \\
    \cmidrule(l){2-7}
    $\alpha$ & 0.1 & 0.25 & 0.5 & 1 & 1.5 & 2 \\
    \cmidrule(l){1-1} \cmidrule(l){2-7}
    $10^{-2}$ & 1.30 & 1.89 & 3.39 & 9.24 & 20.43 & 37.21 \\
    $10^{-3}$ & 1.39 & 2.25 & 4.80 & 18.30 & 55.89 & 137.81 \\
    $10^{-4}$ & 1.48 & 2.61 & 6.43 & 32.74 & 132.43 & 428.06 \\
    $10^{-5}$ & 1.56 & 2.97 & 8.33 & 54.75 & 284.71 & 1,175.97 \\
\end{tabular}
}
        \caption{Multiplicative advantage at FPR $\alpha$}
        \label{tab:app-mult-adv}
    \end{subfigure}
    \caption{Worst-case multiplicative advantage of $\mu$-GDP at FPR level $\alpha$, for several $\mu$ values.
    \textbf{\subref{fig:app-tradeoff-multi}}:~Worst-case multiplicative advantage $\nicefrac{\tpr}{\fpr}$ under $\mu$-GDP as a function of FPR $\alpha$.
    \textbf{\subref{tab:app-mult-adv}}:~Multiplicative advantage values $\nicefrac{\tpr}{\fpr}$ for selected $(\alpha, \mu)$ pairs.}
    \label{fig:app-multi-adv}
\end{figure}

\paragraph{Precision at fixed recall.}
\Cref{fig:app-prec} shows the worst-case precision at recall $\bar\beta$ under $\mu$-GDP, for several $\mu$ values and selected recall $\bar\beta$ levels.

\begin{figure}[h]
    \centering
    \begin{subfigure}[c]{0.55\linewidth}
        \centering
        \resizebox{0.9\linewidth}{!}{%
        \begin{tikzpicture}
            \begin{axis}[
                width=8cm, height=6cm,
                xlabel={TPR},
                ylabel={Precision},
                xmin=0, xmax=1, ymin=0.45, ymax=1.02,
                axis line style={draw opacity=0.5},
                enlargelimits=false,
                font={\sffamily},
                xtick={0,0.25,0.5,0.75,1},
                ytick={0.5,0.6,0.7,0.8,0.9,1},
                grid=major,
                grid style={gray!25, thin},
                legend style={draw=none, fill=none, font={\sffamily\footnotesize},
                              at={(0.97,0.97)}, anchor=north east,
                              row sep=-2pt, inner sep=2pt},
                legend cell align=left,
            ]
            \addplot[color=black!55, dashed, thick,
                     domain=0:1, samples=2]
                {0.5};
            \addlegendentry{Precision $= \nicefrac{1}{2}$}
            \addplot[color=seqA, thick] table[x=recall, y=cA] {data/plots/fig5a_precision_vs_recall.dat};
            \addlegendentry{$\mu = 0.1$}
            \addplot[color=seqB, thick] table[x=recall, y=cB] {data/plots/fig5a_precision_vs_recall.dat};
            \addlegendentry{$\mu = 0.25$}
            \addplot[color=seqC, thick] table[x=recall, y=cC] {data/plots/fig5a_precision_vs_recall.dat};
            \addlegendentry{$\mu = 0.5$}
            \addplot[color=seqD, thick] table[x=recall, y=cD] {data/plots/fig5a_precision_vs_recall.dat};
            \addlegendentry{$\mu = 1$}
            \end{axis}
        \end{tikzpicture}}
        \caption{Precision vs.\ recall}
        \label{fig:app-prec-curves}
    \end{subfigure}%
    \hfill
    \begin{subfigure}[c]{0.42\linewidth}
        \centering
        {\sffamily\footnotesize
        \setlength{\tabcolsep}{4pt}
        \renewcommand{\arraystretch}{1.18}
        \begin{tabular}{@{}rcccccc@{}}
    & \multicolumn{6}{c}{$\mu$} \\
    \cmidrule(l){2-7}
    $\bar\beta$ & 0.1 & 0.25 & 0.5 & 1 & 1.5 & 2 \\
    \cmidrule(l){1-1} \cmidrule(l){2-7}
    $10^{-2}$ & 0.567 & 0.667 & 0.809 & 0.958 & 0.994 & 0.999 \\
    $10^{-3}$ & 0.585 & 0.705 & 0.858 & 0.979 & 0.998 & 1.000 \\
    $10^{-4}$ & 0.599 & 0.735 & 0.891 & 0.988 & 0.999 & 1.000 \\
    $10^{-5}$ & 0.611 & 0.759 & 0.914 & 0.993 & 1.000 & 1.000 \\
\end{tabular}
}
        \caption{Precision at recall $\hat \beta$}
        \label{tab:app-prec}
    \end{subfigure}
    \caption{Precision at recall $\tpr = \bar\beta$ under $\mu$-GDP, for several $\mu$ values.
    \textbf{\subref{fig:app-prec-curves}}:~Worst-case precision as a function of recall $\bar\beta$.
    \textbf{\subref{tab:app-prec}}:~Precision values for selected $(\bar\beta, \mu)$ pairs.}
    \label{fig:app-prec}
\end{figure}

\paragraph{Approximate DP.}
\Cref{fig:app-adp} shows the ADP $\varepsilon$ for several $\mu$ values and $\delta$ values.

\begin{figure}[h]
    \centering
    \begin{subfigure}[c]{0.55\linewidth}
        \centering
        \resizebox{0.9\linewidth}{!}{%
        \begin{tikzpicture}
            \begin{axis}[
                width=8cm, height=6cm,
                xlabel={$\delta$},
                ylabel={$\varepsilon$},
                xmin=-15, xmax=0, ymin=0, ymax=9,
                restrict x to domain=-40:1,
                axis line style={draw opacity=0.5},
                enlargelimits=false,
                font={\sffamily},
                xtick={-15,-12,-9,-6,-3,0},
                xticklabels={$10^{-15}$,$10^{-12}$,$10^{-9}$,$10^{-6}$,$10^{-3}$,$1$},
                ytick={0,2,4,6,8},
                grid=major,
                grid style={gray!25, thin},
                legend style={draw=none, fill=none, font={\sffamily\footnotesize},
                              at={(rel axis cs:0.98,0.98)}, anchor=north east,
                              row sep=-2pt, inner sep=2pt},
                legend cell align=left,
            ]
            \addplot[color=seqA, thick] table[x=cA, y=eps] {data/plots/fig6a_log10delta_vs_eps.dat};
            \addlegendentry{$\mu = 0.1$}
            \addplot[color=seqB, thick] table[x=cB, y=eps] {data/plots/fig6a_log10delta_vs_eps.dat};
            \addlegendentry{$\mu = 0.25$}
            \addplot[color=seqC, thick] table[x=cC, y=eps] {data/plots/fig6a_log10delta_vs_eps.dat};
            \addlegendentry{$\mu = 0.5$}
            \addplot[color=seqD, thick] table[x=cD, y=eps] {data/plots/fig6a_log10delta_vs_eps.dat};
            \addlegendentry{$\mu = 1$}
            \end{axis}
        \end{tikzpicture}}
        \caption{Privacy profile, $\varepsilon(\delta)$}
        \label{fig:app-adp-curves}
    \end{subfigure}%
    \hfill
    \begin{subfigure}[c]{0.42\linewidth}
        \centering
        {\sffamily\footnotesize
        \setlength{\tabcolsep}{3pt}
        \renewcommand{\arraystretch}{1.18}
        \begin{tabular}{@{}rcccccc@{}}
    & \multicolumn{6}{c}{$\mu$} \\
    \cmidrule(l){2-7}
    $\delta$ & 0.1 & 0.25 & 0.5 & 1 & 1.5 & 2 \\
    \cmidrule(l){1-1} \cmidrule(l){2-7}
    $10^{-5}$ & 0.34 & 0.93 & 2 & 4.4 & 7.1 & 10 \\
    $10^{-6}$ & 0.4 & 1.1 & 2.3 & 4.9 & 7.8 & 11 \\
    $10^{-7}$ & 0.45 & 1.2 & 2.5 & 5.3 & 8.5 & 12 \\
    $10^{-8}$ & 0.49 & 1.3 & 2.7 & 5.8 & 9.1 & 13 \\
\end{tabular}
}
        \caption{ADP $\varepsilon$ at fixed $\delta$.}
        \label{tab:app-adp}
    \end{subfigure}
    \caption{Privacy profile of $\mu$-GDP for several $\mu$ values.
    \textbf{\subref{fig:app-adp-curves}}:~$\varepsilon$ as a function of $\delta$.
    \textbf{\subref{tab:app-adp}}:~$\varepsilon$ values such that $\mu$-GDP implies $(\varepsilon,\delta)$-DP, for selected $(\mu, \delta)$ pairs.}
    \label{fig:app-adp}
\end{figure}

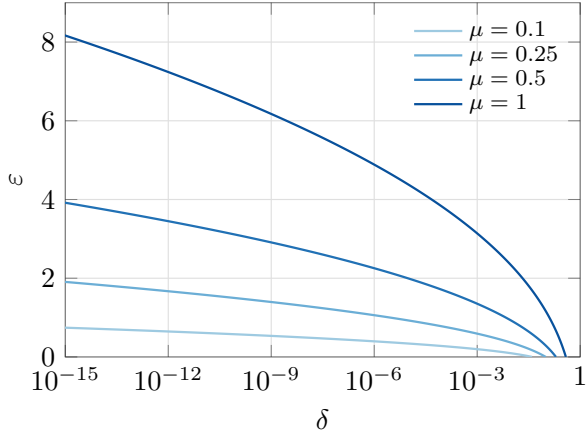
\end{document}